\documentclass[12pt]{article}

\usepackage{amsmath, amsthm, amsfonts, verbatim, amssymb}
\usepackage[margin=1in]{geometry}

\usepackage{bbm}

\usepackage{tikz}
\usetikzlibrary{graphs}
\usetikzlibrary{patterns}

\usepackage{float}

\usepackage[font=scriptsize, margin=1in]{caption}

\usepackage[title]{appendix}

\usepackage{pgfplots}
\pgfplotsset{width=7cm, compat=1.10}
\usepgfplotslibrary{fillbetween}

\usepackage{subcaption}

\usepackage{enumerate}

\usepackage{forest}

\newtheorem{thm}{Theorem}[section]

\newtheorem{lem}[thm]{Lemma}

\theoremstyle{definition}

\theoremstyle{remark}

\def\RR{\mathbb{R}}

\def\QQ{\mathbb{Q}}
\def\AA{\mathcal{A}}

\begin{document}

\title{Identifying the Most Explainable Classifier}

\author{Brett Mullins\\
  \small brettcmullins@gmail.com
}

\maketitle

\begin{abstract}
  We introduce the notion of pointwise coverage to measure the explainability properties of machine learning classifiers. An explanation for a prediction is a definably simple region of the feature space sharing the same label as the prediction, and the coverage of an explanation measures its size or generalizability. With this notion of explanation, we investigate whether or not there is a natural characterization of the most explainable classifier. According with our intuitions, we prove that the binary linear classifier is uniquely the most explainable classifier up to negligible sets.
\end{abstract}

\section{Introduction}

The interpretability of machine learning models and explanations of model predictions have received much attention over the past decade \cite{Gu18}. These approaches attempt to explain what influences a model's behavior on a particular observation \cite{Ri16}. Though these concepts are often equivocated, interpretability usually focuses on what one can learn by inspecting a model's structure, e.g., observing the sign and magnitude of a weight in a linear regression \cite{Mo18}. In contrast, an explanation is a reason for a model's behavior at a specific point in the feature space and is local to that observation \cite{Gi18}. When we talk of a model being interpretable, we mean that its behavior is transparent with respect to inspecting a model's structure, hence the common white-box/black-box dichotomy. When we speak of a model being explainable, we mean that a reason can be given for the model's behavior at a given point in the feature space that is of a sufficient generality for the context of model usage. Explainability is a desirable property for a model, since it allows the user to build trust that the model's predictions accord with background knowledge, to better understand the model's behavior, and to ensure algorithmic fairness when used for potentially consequential decisions \cite{Yi19}.

In this paper, we explore the notion of the most explainable classifier through representing classifiers as partitions of euclidean space. In contrast to algorithmic approaches to measuring explainability, we introduce a theoretical framework where explainability is expressed as a geometric and topological property of a partition of euclidean space. In particular, we adopt the notion of pointwise coverage, first introduced with the probabilistic anchors approach in \cite{Ri18}, as an aggregate measure of explainability over all points in the feature space. Using the notion of pointwise coverage, we prove a characterization result uniquely identifying the most explainable classifier as a refinement of the binary linear classifier. Though this result is unsurprising, it provides a foundation for our intuitions about linear classifiers and corroborates the utility of this theoretical framework.

This paper proceeds as follows. In Section \ref{s:classifiers}, we develop a formal framework to represent classifiers, and, in Section \ref{s:coverage}, we introduce an approach to explanations of classifier predictions and a measure of classifier explainability called pointwise coverage. In Section \ref{s:refined_linear}, we introduce the refined linear classifier and prove that no classifier is more explainable than it with respect to pointwise coverage. In Section \ref{s:infinite_coverage}, we prove the converse result and establish that the refined linear classifier is uniquely the most explainable classifier. In Section \ref{s:refining}, we characterize the collection of classifiers that can be refined to the refined linear classifier. Finally, in Section \ref{s:conclusion}, we conclude.

\section{A Formal Approach to Classifiers} \label{s:classifiers}

In machine learning and related fields, the general task of classification is to accurately assign an observation to its corresponding label. In this section, we present a formal framework to express classifiers as partitions of euclidean space.

We define a \emph{classifier} $P$ as a partition of $\RR^n$ such that there exists $R \in P$, called the \emph{refinement set}, where $R$ is potentially empty, meagre, and Lebesgue null. Define the \emph{label set} of $P$ as $L_P = P \setminus \{ R \}$ and the \emph{feature space} of $P$ as $\bigcup L_P$.
For $x \in \RR^n$, let $P(x) \in P$ be the member of $P$ containing $x$. We call $P(x)$ the \emph{label} of $x$ with respect to model $P$. We call a classifier \emph{trivial} if the label set is a singleton set; otherwise, the classifier is \emph{non-trivial}.

We specify that the refinement set $R \subset \RR^n$ is both meagre and Lebesgue null to capture that $R$ is small or negligible both topologically and probabilistically. This specification follows from the intuition in the case where $R$ is the boundary between two labels of a classifier, e.g., if $R$ is the hyperplane separating the two labels of a binary linear classifier. Recall that a set is \emph{meagre} if it can be represented as the countable union of nowhere dense sets and a set $B$ is \emph{nowhere dense} if the closure of $B, \overline B,$ has no non-trivial open set. On the other hand, for probability measure $\mu$, we have that $\mu(R) = 0$, since $\mu(R) = \int_R f d\lambda$ and $\lambda(R) = 0$ where $\lambda$ is the Lebesgue measure and $f$ is a density function over $\RR^n$.

Observe that we make the simplifying assumption that the data generation process is continuous; while this is not strictly general, it is reasonable given that discrete features are embedded in $\RR^n$ and treated as numerically continuous in many popular machine learning models and algorithms.

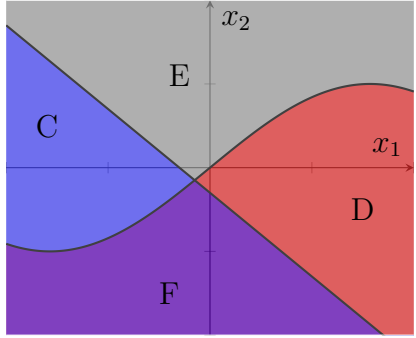
\begin{figure}[h]
  \centering
  \begin{tikzpicture}
    \begin{axis}[ xlabel={$x_1$}, ylabel={$x_2$}
      ,axis lines=middle
      ,samples=200, thick
      ,domain=-20:20
      ,xmin=-20
      ,xmax=20
      ,ymin=-20
      ,ymax=20
      ,xticklabels={,,}
      ,yticklabels={,,}
      ]
    \addplot+[no marks, name path=l1, color=darkgray, solid, thick] {10*sin(0.1*deg(x))};
    \addplot+[no marks, name path=l2, color=darkgray, solid, thick] {-1*x - 3};
    \addplot+[no marks, name path=above, color=white] {20};
    \addplot+[no marks, name path=below, color=white] {-20};
    \addplot[lightgray, opacity = 0.5] fill between[of=l1 and above];
    \addplot[gray, opacity = 0.5] fill between[of=l2 and above];
    \addplot[red, opacity = 0.5] fill between[of=l1 and below];
    \addplot[blue, opacity = 0.5] fill between[of=l2 and below];
    \node at (axis cs:-16,5) {C};
    \node at (axis cs:15,-5) {D};
    \node at (axis cs:-3,11) {E};
    \node at (axis cs:-4,-15) {F};
    \end{axis}
  \end{tikzpicture}
  \caption{Example Classifier as a Partition of $\mathbb{R}^2$}
  \label{g:curvy_example}
\end{figure}

As an example, let us consider the classifier $P$ in Figure \ref{g:curvy_example}. For this classifier, $L_P = \{ C, D, E, F \}$. From this figure, we can consider multiple classifiers. One such classifier is given by $P = L_P \cup \{ \emptyset \}$ where $E, F$ contain their boundaries but $C, D$ do not. In this case, the refinement set is empty. A classifier with an empty refinement set is called \emph{ordinary}; otherwise, a classifier is a \emph{refinement} of some ordinary classifier and is called \emph{refined}.
An example of a refined classifier from the figure above is given by $Q = \{ C, D, E, F, R \}$ where $L_Q = \{ C, D, E, F \}$, no member of $L_Q$ contains their boundary, and $R = \bigcup_{L \in L_Q} bd(L)$.\footnotemark While we only ever see ordinary classifiers ``in the wild'', the introduction of the notion of refinements removes the artificial complexity generated by edge cases in the feature space. By moving to a refinement of a classifier, we can better assess the aggregate explainability and topological properties of that classifier when edge cases are present.

\footnotetext{Note that $bd(X) = \overline{X} \cap \overline{\RR^n \setminus X}$ is the boundary of the set $X \subseteq \RR^n$. Whenever possible we follow the notation conventions in \cite{Mu00}.}

The framework developed in this section is sufficiently general to represent any classifier with continuous features. This ranges from a binary linear classifier where the labels are open and closed halfspaces in the feature space, respectively, to decision trees where the labels are disjoint unions of convex polytopes, i.e., intersections of open and closed halfspaces. These models are well-studied and have simple geometric characterizations. As we increase complexity with, for example, neural networks, we find that representations of these classifiers within this framework are possible but not intuitive or clear due to the compositions of non-linear activation functions found in many neural network architectures \cite{Sh14}. Nonetheless, in the next section, we introduce the pointwise coverage approach to measuring the explainability of a classifier expressed as a partition of euclidean space in the framework developed thus far.

\section{Explainability by Pointwise Coverage} \label{s:coverage}

An explanation of a classifier prediction is an elusive concept. Ideally, an explanation provides a reason for why a classifier assigns a particular label to a given observation. One way to achieve this and the perspective we adopt in this paper identifies an explanation for a classifier at a given observation as a definably simple region of the feature space containing the observation where all points in the region are assigned the same label. We refer to these definably simple regions of the feature space as anchors.

To what extent does defining such a region of the feature space provide an explanation for the model's classification? At first pass, we can think of an anchor as a sufficient condition for the classification of a point in the feature space; however, that alone is unhelpful, since an anchor could be an arbitrary subset of a label. By adding the requirement that anchors be definably simple regions of the feature space, we can ensure that the points in the anchor are meaningfully related or related by a simple condition. An anchor for a given point acts as an explanation by providing the definition of the relation grouping the points in the anchor as the reason for the classification.

Just as there are many approaches to interpretability and explainability, there are many ways to specify what is meant for an anchor to be definably simple. As an example, the probabilistic anchors approach uses rectangles in the feature space that minimize the number of conditions specified as anchors \cite{Ri18}. In contrast, we adopt open balls in euclidean space as anchors. Observe that both of these approaches use a distance-based relation to group points in the respective anchors, so that points in an anchor are in some sense spatially close to one another. While the probabilistic anchors approach is largely concerned with algorithmic and computational properties of identifying anchors in the feature space \cite{Si16,Gu18a}, we focus on geometric and topological properties of the label set. In particular, for a classifier $P$, we define an \emph{anchor} for a point $x \in \RR^n$ as an open ball $A = B(c,r)$ such that $x \in A \subset P(x)$. Notice that the anchor need not be centered at the observation of interest.

Observe that open balls are basic open sets in the standard topology on $\RR^n$. From a definability perspective, basic open sets are among the most simple sets of a topology, since all other open sets are countable unions of basic open sets. Moreover, the open sets occupy the space at the bottom of the Borel hierarchy, a stratification of the Borel sets, i.e., the sets constructed from open sets by iterative application of countable union, countable intersection, and complementation, and ordered by their definability in terms of open sets. To denote that the open sets are definably simple, we say that the open sets have a \emph{Borel rank} of 1. Sets of greater Borel rank are then more definably complex.\footnotemark

\footnotetext{Sets of greater Borel rank are outside of the scope of the present paper. For more on the Borel sets and the Borel Hierarchy, see \cite{Ke95, Sr98}.}

Explanations are usually local in the sense that they do not apply to all points in the feature space. For example, an explanation for the classification of a point $x \in \RR^n$ by classifier $P$ need not be an explanation for the classification of a distinct point $y \in \RR^n$. In particular, this will be the case when $y$ does not belong to an anchor for $x$ irrespective of whether $P(y) = P(x)$ or not. We may make the notion of local explanations precise by introducing the coverage of an anchor. For an anchor $A$ for point $x \in \RR^n$ with radius $r_A$ with respect to classifier $P$, the \emph{coverage of $A$} is given by $c_P(A) = r_A > 0$. If there exists an anchor $A = B(c,r)$ for a point $x \in \RR^n$ with radius $r > 0$ then that anchor acts as an explanation for all points in the feature space within an $r-$neighborhood of $c$.

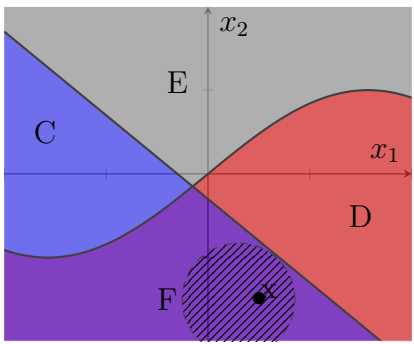
\begin{figure}[h]
  \centering
  \begin{tikzpicture}
    \begin{axis}[ xlabel={$x_1$}, ylabel={$x_2$}
      ,axis lines=middle
      ,samples=200, thick
      ,domain=-20:20
      ,xmin=-20
      ,xmax=20
      ,ymin=-20
      ,ymax=20
      ,xticklabels={,,}
      ,yticklabels={,,}
      ]
    \addplot+[no marks, name path=l1, color=darkgray, solid, thick] {10*sin(0.1*deg(x))};
    \addplot+[no marks, name path=l2, color=darkgray, solid, thick] {-1*x - 3};
    \addplot+[no marks, name path=above, color=white] {20};
    \addplot+[no marks, name path=below, color=white] {-20};
    \addplot[lightgray, opacity = 0.5] fill between[of=l1 and above];
    \addplot[gray, opacity = 0.5] fill between[of=l2 and above];
    \addplot[red, opacity = 0.5] fill between[of=l1 and below];
    \addplot[blue, opacity = 0.5] fill between[of=l2 and below];
    \node at (axis cs:-16,5) {C};
    \node at (axis cs:15,-5) {D};
    \node at (axis cs:-3,11) {E};
    \node at (axis cs:-4,-15) {F};
    \node at (axis cs:5,-15) {\textbullet};
    \node at (axis cs:6,-14) {x};
    \draw[darkgray,dashed,pattern=north east lines] (axis cs:3,-15) circle (0.75cm);
    \end{axis}
  \end{tikzpicture}
  \caption{Example Classifier with Anchors}
  \label{g:anchors_example}
\end{figure}

If a point in the feature space has an anchor, then it has many such anchors. To see this, observe that if $A$ is an anchor for $x \in \RR^n$ then there exists an $r > 0$ such that $B(x, r) \subset A$. Clearly, $B(x, r)$ is an anchor for $x$; however, $c_P(B(x, r)) = r \leq c_P(A)$, since $B(x, r) \subset A$. Given that a point can have many anchors and each are equally definably simple from a topological perspective, we choose an anchor with the greatest coverage as the best explanation for the classification.
Let $\AA_x$ denote the set of anchors for $x$ with respect to classifier $P$. We say that the \emph{coverage of $P$ at $x$} is given by $C_P(x) = \sup_{A \in \AA_x}c_P(A)$. If no anchors exist for a point $x$ with respect to classifier $P$, we say that the coverage of $P$ at $x$ is zero. If there exists a sequence of anchors for $x$ with increasing unbounded coverage, then we say that the coverage of $P$ at $x$ is infinite. Otherwise, we say that the coverage of $P$ at $x$ is finite.

Why do we prefer anchors with greater coverage to those with less as explanations? Just as coverage is a measure of the size of an anchor as a ball in euclidean space, it is also a measure of the generality of an explanation in feature space. In turn, we may say that coverage is a measure of the strength of a reason for a classifier's behavior at a point in the feature space. To illustrate this point, let us consider a concern raised about the veracity of explanations of classifier predictions when a classifier learns spuriously or erroneously from its training data \cite{La19}.

\begin{figure}[h]
  \centering
  \begin{tikzpicture}
    \begin{axis}[ xlabel={$x_1$}, ylabel={$x_2$}
      ,axis lines=middle
      ,samples=200, thick
      ,domain=-20:20
      ,xmin=-20
      ,xmax=20
      ,ymin=-20
      ,ymax=20
      ,xticklabels={,,}
      ,yticklabels={,,}
      ]

    \addplot[fill = blue!40!white, color = blue!40!white, opacity = 0.5] coordinates {
      (-7,1) (-7,20) (20, 20) (20,1)
    };
    \addplot[fill = blue!40!white, color = blue!40!white, opacity = 0.5] coordinates {
      (-20,-1) (-20,-10) (18, -10) (18,-1)
    };
    \addplot[fill = red!40!white, color = red!40!white, opacity = 0.5] coordinates {
      (-20,1) (-20,-1) (20, -1) (20,1)
    };
    \addplot[fill = red!40!white, color = red!40!white, opacity = 0.5] coordinates {
      (18,-1) (18,-10) (20, -10) (20,-1)
    };
    \addplot[fill = red!40!white, color = red!40!white, opacity = 0.5] coordinates {
      (-20,-10) (-20,-20) (20, -20) (20,-10)
    };
    \addplot[fill = red!40!white, color = red!40!white, opacity = 0.5] coordinates {
      (-20,1) (-20,20) (-7, 20) (-7,1)
    };
    \node at (axis cs:10,10) {M};
    \node at (axis cs:-10,-15) {N};
    \node at (axis cs:5,0) {\textbullet};
    \node at (axis cs:6,1) {x};
    \node at (axis cs:-15,10) {\textbullet};
    \node at (axis cs:-14,11) {y};

    \end{axis}
  \end{tikzpicture}
  \caption{Example of an Overfit Decision Tree}
  \label{g:overfit_classifier}
\end{figure}
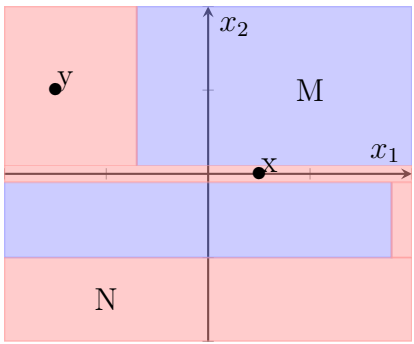

For the former case, consider a decision tree classifier that is potentially overfit during training. The example in Figure \ref{g:overfit_classifier} is a binary classifier $P = \{ M, N \}$, and let us assume that it is overfit with respect to the region containing the point $x$. Notice that since each label is the disjoint union of convex polytopes, this classifier represents a decision tree. Let us consider the coverage of $P$ at points $x, y$ in the feature space. While both points belong to the label $N$, it is apparent that $C_P(x) < C_P(y)$. The spurious learning of the classifier is, thus, reflected in this disparity in coverage between the two points. Relative to $y$, the explanation for the classification for $x$ is much weaker. To address the latter case, there need not be a correlation between a classifier's explainability and its veracity with respect to the training data. This is to say that the explainability properties of a classifier do not necessarily imply anything about it correctly learning from the training data.

To compare the explainability of a classifier at two points in the feature space, we can compare the classifier's coverage at those points. Note that this comparison is always relative to the scale of the feature dimensions; if the scale of the features are transformed, e.g., by an affine transformation, then the resulting coverage values may be different, since what were previously anchors may now be open ellipsoids rather than open balls. Fixing the scale of features, on the other hand, permits comparisons of coverage and, resultantly, explainability for particular points in the feature space across various classifiers.

Intuitively, when one estimates the explainability of a classifier or compares the explainability of multiple classifiers, it is not with reference to a specific point in the feature space. Comparing the coverages of various classifiers at every point in the feature space is not feasible, since that would entail uncountably many comparisons with no clear method of aggregating or summarizing the results of the comparison. A simple method of aggregating coverage up to the classifier-level is to consider the infimum and supremum of coverage across all points in the feature space. We refer to these aggregations of coverage as \emph{pointwise coverage}.

Let us consider two limiting cases of pointwise coverage: zero pointwise coverage and infinite pointwise coverage. We say that a classifier has \emph{zero pointwise coverage} if the supremum across all points in the feature space of a classifier's coverage is zero. With respect to coverage, classifiers with zero pointwise coverage are the least explainable; no point in the feature space has an anchor. Let us provide an example of such a classifier. Consider a classifier on $\RR$ given by $P = \{ \QQ, \RR \setminus \QQ, \emptyset \}$. Since both labels are dense in $\RR$, any potential anchor for a point in $\QQ$ must contain a point in $\RR \setminus \QQ$, and vice-versa. Luckily, one is almost surely not to encounter such an unexplainable classifier ``in the wild.''

Whereas zero pointwise coverage represents a lower-bound on explainability, infinite pointwise coverage is an upper-bound. We say that a classifier has \emph{infinite pointwise coverage} if the infimum across all points in the feature space of a classifier's coverage is infinite. In Section \ref{s:refined_linear}, we prove that there is a natural collection of classifiers that have this property: a refinement of binary linear classifiers which we call refined linear classifiers. Moreover, in Section \ref{s:infinite_coverage}, we prove that only the collection of refined linear classifiers have infinite pointwise coverage. The primary result of this paper is a full characterization of infinite pointwise coverage as a classifier of the form of a refined linear classifier:
\begin{thm} \label{thm:refined_linear_iff}
  A non-trivial classifier $P$ has infinite pointwise coverage just in case $P$ is a refined linear classifier.
\end{thm}

\section{Refined Linear Classifier} \label{s:refined_linear}

Linear classifiers are ubiquitous throughout the history of machine learning and in data science today. ``The family of linear [classifiers] is one of the most useful families of hypothesis classes, and many learning algorithms that are being widely used in practice rely on linear [classifiers]'' \cite{Sh14}. This collection of classifiers includes not just linear regression but logistic regression, perceptrons \cite{Ro58}, linear support vector machines \cite{Va98}, etc.\footnotemark Within the scope of this paper, we are interested in binary linear classifiers, i.e. a linear classifier with only two labels.

\footnotetext{For more on linear classifiers, see \cite{Us14}, Chapter 9 of \cite{Sh14}, and Chapter 4 of \cite{Ha01}.}

Despite the algorithm or method used to train the classifier, a binary linear classifier can always be represented as the weighted sum of input features with real-valued weights and a real-valued threshold. With respect to pointwise coverage, we are interested in the geometric and topological characterization of a classifier rather than its explicit functional form. A \emph{binary linear classifier} is a classifier of the form $P = \{ M, N, R \}$ where $R$ is empty and $M, N$ are the open and closed halfspaces, respectively. Figure \ref{g:linear_classifier} below is an example of a binary linear classifier.

\begin{figure}[h]
  \centering
  \begin{tikzpicture}
    \begin{axis}[ xlabel={$x_1$}, ylabel={$x_2$}
      ,axis lines=middle
      ,style = {black}
      ,samples=41, thick
      ,domain=-20:20
      ,xmin=-20
      ,xmax=20
      ,ymin=-10
      ,ymax=10
      ,xticklabels={,,}
      ,yticklabels={,,}
      ]
    \addplot+[no marks, name path=decision, color=lightgray, solid, thick] {0.5*x - 1};
    \addplot+[no marks, name path=above, color=white] {10};
    \addplot+[no marks, name path=below, color=white] {-10};
    \addplot[white, opacity = 0.5] fill between[of=decision and above];
    \addplot[lightgray, opacity = 0.5] fill between[of=decision and below];
    \node at (axis cs:-10,5) {M};
    \node at (axis cs:10,-5) {N};
    \end{axis}
  \end{tikzpicture}
  \caption{Example Binary Linear Classifier}
  \label{g:linear_classifier}
\end{figure}
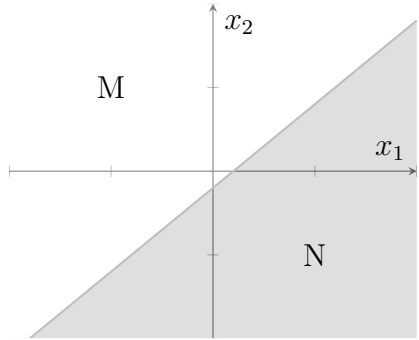

Let us consider the coverage properties of a binary linear classifier. By inspection, it is apparent that many points have non-zero coverage. Recall that the \emph{decision boundary} of a classifier is the set of points that separate labels. Put another way, the decision boundary for classifier $P$ with label set $L_P$ is given by $\bigcup_{L \in L_P} bd(L)$ as with the example accompanying Figure \ref{g:curvy_example} in Section \ref{s:classifiers}. For the binary linear classifier, its decision boundary is the hyperplane bordering the two labels. Note that the decision boundary actually belongs to one to the two labels, since one label is an open halfspace and the other is a closed halfspace. As a result, the classifier has zero coverage at each point on the decision boundary.

Observe that the set of points on which the binary linear classifier has zero coverage is Lebesgue null. Moreover, since the decision boundary is the boundary of an open set, it is nowhere dense, implying that the decision boundary is meagre. Let us introduce a refinement of the binary linear classifier by moving the decision boundary from the feature space to the refinement set. A \emph{refined linear classifier} is a classifier of the form $P = \{ M, N, R \}$ where $R$ is a hyperplane and $M, N$ are the open halfspaces above and below $R$. An example is illustrated with Figure \ref{g:refined_linear} below.

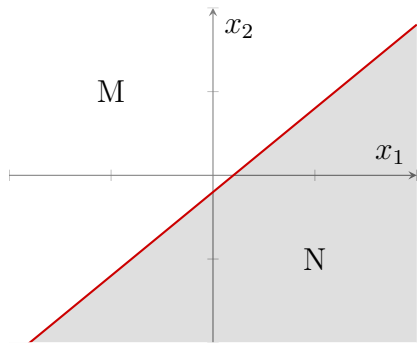
\begin{figure}[h]
  \centering
  \begin{tikzpicture}
    \begin{axis}[ xlabel={$x_1$}, ylabel={$x_2$}
      ,axis lines=middle
      ,samples=41, thick
      ,domain=-20:20
      ,xmin=-20
      ,xmax=20
      ,ymin=-10
      ,ymax=10
      ,xticklabels={,,}
      ,yticklabels={,,}
      ]
    \addplot+[no marks, name path=decision, color=red!80!black, solid, thick] {0.5*x - 1};
    \addplot+[no marks, name path=above, color=white] {10};
    \addplot+[no marks, name path=below, color=white] {-10};
    \addplot[white, opacity = 0.5] fill between[of=decision and above];
    \addplot[lightgray, opacity = 0.5] fill between[of=decision and below];
    \node at (axis cs:-10,5) {M};
    \node at (axis cs:10,-5) {N};
    \end{axis}
  \end{tikzpicture}
  \caption{Example Refined Linear Classifier}
  \label{g:refined_linear}
\end{figure}

By moving from the binary linear classifier to the refined linear classifier, the classifier no longer has zero coverage at any point in the feature space. Moreover, Theorem \ref{p:refined_linear} demonstrates that the refined linear classifier has infinite coverage at every point in the feature space.

\begin{thm} \label{p:refined_linear}
  If $P$ is a refined linear classifier, then $P$ has infinite pointwise coverage.
\end{thm}

\begin{proof}
  Suppose $P = \{ M, N, R \}$ is a refined linear classifier with $L_P = \{ M, N \}$ where $M, N$ are open halfspaces and $R$ is a hyperplane. Let $x \in \bigcup L_P$. Let $O$ be a line containing $x$ and orthogonal to $R$. Without loss of generality, we may assume that $O \ \cap \ R = \bf{0}$, i.e., the origin. Let $O^+ = \{ y \in O | \Vert y \Vert > \Vert x \Vert \} \cap P(x)$.
  Let $o_1, o_2, \ldots$ be a sequence of points on $O^+$ that are increasingly far from $x$, i.e., $\Vert o_i \Vert < \Vert o_{i+1} \Vert$. Let $\alpha < \Vert x \Vert$. Define $A_i = B(o_i, r_i)$, where $r_i = \Vert o_i - x \Vert + \alpha$.
  Observe that $A_i \subset P(x)$, since $r_i < \Vert o_i \Vert$. Since $x \in A_i, i \geq 1$, each $A_i$ is an anchor for $x$. Given that $(r_i)_{i \geq 1}$ is an increasing sequence, we have that $C_P(x) = \infty$.
\end{proof}

\section{Infinite Pointwise Coverage} \label{s:infinite_coverage}

In this section, we prove the inverse of Theorem \ref{p:refined_linear}: if a non-trivial classifier has infinite pointwise coverage, then it is a refined linear classifier. To attain this result, we first prove the following three lemmas. The first is a geometric property of an unbounded sequence of balls with at least a single point in common. The second applies the first lemma to the case of infinite coverage at a point in the feature space to imply properties about the shape and size of the point's label. Finally, the third proves that if a classifier has infinite pointwise coverage then it can have at most two labels in the label set.

\begin{lem} \label{p:open_half_space}
  If $x \in \RR^n$ and $(B_i)_{i \geq 1}$ is an unbounded sequence of balls with each containing $x$, $B = \bigcup_{i \geq 1} B_i$ contains an open halfspace $H$ such that $x \in bd(H)$.
\end{lem}

\begin{proof} [Proof (with Paul Larson)]
  Let $x \in \RR^n$ and $B_n = B(q_n, r_n)$ be a ball centered at $q_n$ with radius $r_n > n$ such that $x \in B_n, n \geq 1$. Without loss of generality, let us assume that $x = \bf{0}$, the origin in $\RR^n$, and that $(r_n)_{n \geq 1}$ is strictly increasing. For each $n \geq 1$, let $s_n$ to be the unique point on the line between $\bf{0}$, $ q_n$ such that $\Vert s_n \Vert = 1$.
  Since the unit sphere $S^n \subset \RR^n$ is compact, let us fix $s^* \in S^n$ and suppose $s_n \rightarrow s^*$. By rotating space, we may assume that $s^* = (1, 0, \ldots, 0)$.

  Let $H = \{ (x_1, x_2, \ldots, x_n) | x_1 > 0 \}$ and $\theta_p$ is the angle $p \textbf{0} s^*$ for $p \in H$. Observe that $H$ is an open halfspace and $\cos \theta_p > 0$ for $p \in H$. While $\textbf{0} \notin H$, we have that $\textbf{0} \in bd(H)$.

  We want to show that $H \subset B = \bigcup_{i \geq 1} B_i$.
  Let $d_n = r_n - \Vert q_n \Vert$. If $d_n$ is unbounded, then $B = \RR^n$. Clearly, $H \subset B$. Otherwise, suppose $d_n$ is bounded. Since $r_n$ is increasing and unbounded, $\Vert q_n \Vert$ increases to infinity. For $p \in H$, it is sufficient to show that $\Vert q_n \Vert > \Vert q_n - p \Vert$, since $B_n$ contains $\textbf{0}$, i.e., $\Vert q_n \Vert < r_n$.
  This distance condition is equivalent to $\Vert p \Vert < 2 \Vert q_n \Vert \cos \theta_{n,p}$, where $\theta_{n,p}$ is the angle $p \textbf{0} s_n$. Since $\Vert q_n \Vert$ increases to infinity and $\theta_{n,p} \rightarrow \theta_p > 0$ , we can find a sufficiently large $n$ such that the distance condition holds. For such an $n$, $p \in B_n \subset B$.
\end{proof}

Lemma \ref{p:open_half_space} belongs to a family of results connecting the structure of spaces to the properties of unbounded sequences of balls \cite{Ba01}.

\begin{lem}  \label{p:label_open_hspace}
  If $P$ is a classifier such that the coverage of $P$ at $x$ is infinite, then $P(x)$ contains an open halfspace $H$ such that $x \in bd(H)$.
\end{lem}

\begin{proof}
  Suppose $P$ is a classifier and the coverage of $P$ at $x$ is infinite. Then there exists a sequence of anchors for $x$ $(A_i)_{i \geq 1}$ with unbounded coverage. Let $A = \bigcup_{i \geq 1} A_i$. Observe that $A \subset P(x)$. By Lemma \ref{p:open_half_space}, $H \subset A$, where $H$ is an open halfspace. Hence, $H \subset P(x)$.
\end{proof}

\begin{lem} \label{p:inf_gc_PA}
  A classifier with infinite pointwise coverage can have at most two labels in the label set.
\end{lem}
\begin{proof}
  Let $P$ be a classifier with infinite global coverage. Let $x, y$ be such that $P(x), P(y) \in L_P$ and $P(x) \neq P(y)$. By Lemma \ref{p:label_open_hspace}, there exists open halfspaces $H_x \subset P(x)$ and $H_y \subset P(y)$. Observe that since $P(x), P(y)$ are disjoint, $bd(H_x), bd(H_y)$ must be parallel; otherwise, $H_x \cup H_y$ is non-empty.
  For contradiction, suppose there is a $z \notin P(x) \cup P(y)$ where $P(z) \in L_P$. Applying Lemma \ref{p:label_open_hspace} once more, we obtain that there is an open halfspace $H_z \subset P(z)$. Thus, we have $\RR^n$ contains three disjoint open halfspaces $H_x, H_y, H_z$.
  Then $bd(H_z)$ must be parallel to $bd(H_x), bd(H_y)$; otherwise, the intersection of $H_z$ with each of $H_x, H_y$ is non-empty. Observe that the halfspace to one side of $bd(H_z)$ has a non-empty intersection with $H_x$, while the other side has a non-empty intersection with $H_y$. $(\rightarrow \leftarrow)$.
\end{proof}

Observe that a classifier with a single label can have infinite pointwise coverage; however, a trivial classifier does not necessarily have infinite pointwise coverage. We provide an example of such a classifier in Section \ref{s:refining}. With that being said, we know that there is a unique ordinary trivial classifier, namely $\{ \RR^n, \emptyset \}$, and this classifier has infinite pointwise coverage.

With these lemmas in hand, we may now turn to the main proof in this section.

\begin{thm} \label{p:inf_cov->refined_linear}
  If $P$ is a non-trivial classifier with infinite pointwise coverage, then $P$ is a refined linear classifier.
\end{thm}
\begin{proof}
  Let $P$ be a non-trivial classifier with infinite pointwise coverage. Let $x \in P(x) \in L_P$. Since $P$ has infinite pointwise coverage, the coverage of $P$ at $x$ is infinite. By Lemma \ref{p:label_open_hspace}, $H_x$ is an open halfspace such that $H_x \subset P(x)$ and $x \in bd(H_x)$.
  Since $P$ is non-trivial, there is a $y \notin P(x)$ where $P(y) \in L_P$. By Lemma \ref{p:label_open_hspace}, $H_y \subset P(y)$ is an open halfspace with $y \in bd(H_y)$. Observe that since $P$ is a partition, $P(x) \cap P(y) = \emptyset$.
  By Lemma \ref{p:inf_gc_PA}, we have that $L = \{ P(x), P(y) \}$. Moreover, since $P(x), P(y)$ are disjoint, $bd(H_x), bd(H_y)$ must be parallel; otherwise, $H_x \cup H_y$ is non-empty.

  Let us define $H_{P(x)} = \bigcup_{z \in P(x)} H_z$, where $H_z$ refers to the open halfspace generated by applying Lemma \ref{p:label_open_hspace} to $z$. We claim that $H_{P(x)} = P(x)$. On the one hand, suppose $w \in H_{P(x)}$. Then, for some $z \in P(x)$, $w \in H_z$. Since $H_z \subset P(x)$, by construction, $w \in P(x)$.
  On the other hand, suppose instead that $w \in P(x)$. Since $P$ has infinite pointwise coverage, $P$ has infinite coverage at $w$. Let $A$ be an anchor for $w$. Let us define both $O_w$ as the line containing $w$ and orthogonal to $bd(H_y)$ and $\hat{y} = O_w \cap bd(H_y)$ as the single point common to both $O_w$ and the hyperplane $bd(H_y)$.
  Choose a point $\hat{w}$ from $A \cap O_w$ where $\Vert w - \hat{y} \Vert > \Vert \hat{w} - \hat{y} \Vert$.
  Applying Lemma \ref{p:label_open_hspace} to $\hat{w}$, we obtain an open halfspace $H_{\hat{w}} \subset P(x)$ where $bd(H_{\hat{w}})$ is parallel to $bd(H_y)$. By construction, $w \in H_{\hat{w}} \subset H_{P(x)}$.

  We have established that $P(x)$ is the countable union of open halfspaces. We further claim that $P(x)$ is an open halfspace. Let us first note that $P(x)$ is open since it is the union of open sets. For each $z \in P(x)$, $bd(H_z)$ must be parallel to $bd(H_y)$, implying $bd(H_z)$ is parallel to $bd(H_{z'})$ for $z, z' \in P(x)$. Then $H_z \cup H_{z'}$ is either $H_z$ or $H_{z'}$, i.e., a halfspace, and so forth.

  With $P(x)$ being an open halfspace, we also have that $P(y)$ is an open halfspace by symmetry. Recall that $R$ is meagre by assumption and that $bd(P(x)) \cup bd(P(y)) \subset \ R$. Then $R$ contains no non-trivial open set. For contradiction, let us suppose that $bd(P(x)) \neq bd(P(y))$. Let $O$ be a line orthogonal to both $bd(P(x)), bd(P(y))$; such a line exists, because $bd(P(x)), bd(P(y))$ are parallel.
  Let $\bar{x} = O \cap bd(P(x))$ and $\bar{y} = O \cap bd(P(y))$. Define $\alpha < \Vert \bar{x} - \bar{y} \Vert$ and $r = \frac{\bar{x} + \bar{y}}{2}$. Then $B(r, \alpha) \subset \ R$, but $R$ is contains no non-trivial open set. $(\rightarrow \leftarrow)$.

  We obtain $P = \{ P(x), P(y), R \}$ where $P(x), P(y)$ are open halfspaces and $R$ is the closed meagre hyperplane separating the halfspaces. Hence, $P$ is a refined linear classifier.
\end{proof}

Theorem \ref{p:inf_cov->refined_linear} provides a link between the pointwise coverage of a classifier and the geometry and topology of its labels. In particular, if a classifier has infinite pointwise coverage, then its labels consist of two open halfspaces. This result accords with our intuitions about infinite coverage and the curvature of the decision boundary. Namely, if the decision boundary is curved, then on one side of the boundary, for some point in that label, anchors will be bounded in size. Moreover, along with Theorem \ref{p:refined_linear}, Theorem \ref{p:inf_cov->refined_linear} implies the primary result of this paper: Theorem \ref{thm:refined_linear_iff}, a characterization of the most explainable collection of classifiers.

An obvious corollary of this result is that no ordinary classifier has infinite pointwise coverage. In Section \ref{s:refining}, we explore the collection of classifiers that can be refined to the refined linear classifier, i.e., to a classifier with infinite pointwise coverage.

\section{Refining Ordinary Classifiers} \label{s:refining}

In Sections \ref{s:refined_linear} and \ref{s:infinite_coverage}, we established that the refined linear classifier is uniquely the most explainable classifier. While this is an interesting property of the pointwise coverage framework, in isolation, interest in this result is limited to strictly theoretical concerns. Recall that only ordinary classifiers are found ``in the wild''. To this end, in Section \ref{s:refined_linear}, we illustrated that a refined linear classifier is a refinement of a binary linear classifier, a typical ordinary classifier. In this section, we identify the collection of ordinary classifiers that can be refined to a refined linear classifier: the generalized binary linear classifiers.

It is worth taking a moment to reflect on what is a refinement of a classifier. In the examples provided thus far, refinements have been used to increase pointwise coverage of a classifier by removing edge cases from the feature space, particularly those along the decision boundary between the labels. For example, while the refined linear classifier has infinite pointwise coverage, the binary linear classifier does not; its edge cases have zero coverage. By moving to the refined model in the case of the linear classifier, we are able to measure and aggregate the explainability of the model without interference from edge cases.

It is not the case, however, that all refinements improve or even preserve pointwise coverage. In fact, for any classifier, there exists a refinement that has zero pointwise coverage. Recall that a classifier has zero pointwise coverage if the supremum of coverage over all points in the feature space is zero. To see this, let $P$ be a classifier on $\RR^n$. Observe that $\QQ^n$ is meagre, Lebesgue null, and dense in $\RR^n$, i.e., $\overline {\QQ^n} = \RR^n$. Let $P'$ be a refinement of $P$ by moving $\QQ^n$ from the feature space of $P$ to the refinement set. Since $\QQ^n$ is dense in $\RR^n$, an open ball containing any particular point $x$ in the feature space will also contain a point in $\QQ^n \subset R$. Hence, there are no anchors for $x$, so $P'$ has zero pointwise coverage. Just as there are unexplainable ordinary classifiers, there are unexplainable unexplainable refined classifiers.

Let us introduce some helpful terminology for refinements. A label is called \emph{negligible} if it is both meagre and Lebesgue null. We say that a refinement $P'$ of classifier $P$ is \emph{eliminative} if $L_{P'}$ is a strict subset of $L_P$. Observe that a classifier $P$ has an eliminative refinement just in case $P$ contains a negligible label.

Let us suppose that $P = \{M, N, R \}$ is a refined linear classifier where $M, N$ are open halfspaces and $R$ is their separating hyperplane and $Q$ is an ordinary classifier with no negligible labels such that $P$ is a refinement of $Q$. Which classifiers satisfy the conditions of $Q$? Since $Q$ has no negligible labels, $P$ is not an eliminative refinement. Then there exists $A \in L_Q$ such that $M \subset A$ and $B \in L_Q$ such that $N \subset B$. Moreover, since $Q$ is ordinary, its refinement set is empty, implying $A \cup B = M \cup N \cup R$.
It may be that $R \subset A$ or $R \subset B$, in which case $Q$ is a binary linear classifier. Additionally, it may be the case that some of $R$ belongs to $A$ and the rest of $R$ belongs to $B$. By removing the constraint that $Q$ has no negligible labels, we extend the collection of ordinary classifiers to include those which partition $R$ into arbitrarily many, even countably many, labels. We refer to this collection as the \emph{generalized binary linear classifiers}.

From the perspective of pointwise coverage, all generalized binary linear classifiers are equivalent: they can be refined to the refined linear classifier. There is a sense in which the binary linear classifier is more natural than the other generalized binary linear classifiers; however, that is outside of the scope of the pointwise coverage framework and is presently left to heuristic. Let us conclude this section by noting that a generalized binary linear classifier is equivalent to a binary linear classifier up to null sets, since the decision boundary is a Lebesgue null set.

\section{Conclusion} \label{s:conclusion}

By introducing a formal framework for classifiers and the topological notion of pointwise coverage, we are able to express what is meant by the most explainable classifier, infinite pointwise coverage, and identify the unique collection of classifiers with this property, the refined linear classifiers. Moreover, up to null sets, only one classifier found ``in the wild'' can be refined to a refined linear classifier: a binary linear classifier. This result accords with our intuitions about the simplicity, utility, and explainability of the binary linear classifier.

%
\bibliographystyle{acm}
\bibliography{imec.bib}

\end{document}